\documentclass{article}

\usepackage{times}
\usepackage{graphicx} 
\usepackage[caption=false]{subfig}
\usepackage{multirow}
\usepackage{natbib}

\usepackage{algorithm}
\usepackage{algorithmic}

\usepackage{enumitem}

\usepackage{hyperref}
\usepackage{mdframed}

\usepackage[accepted]{icml2017} 

\usepackage{verbatim}
\usepackage{blindtext}
\usepackage{amsthm}
\usepackage{amsmath}
\usepackage{amssymb}
\usepackage{framed}
\newtheorem{theorem}{Theorem}[section]

\newtheorem{definition}{Definition}[section]
\newtheorem{lemma}[theorem]{Lemma}
\newtheorem{example}[theorem]{Example}
\newtheorem{remark}[theorem]{Remark}
\newtheorem{proposition}[theorem]{Proposition}

\icmltitlerunning{Estimating the unseen from multiple populations}

\begin{document} 

\twocolumn[
\icmltitle{Estimating the unseen from multiple populations}

\begin{icmlauthorlist}
\icmlauthor{Aditi Raghunathan}{stan}
\icmlauthor{Gregory Valiant}{stan}
\icmlauthor{James Zou}{stan,biohub}
\end{icmlauthorlist}

\icmlaffiliation{stan}{Stanford University, Stanford, CA}
\icmlaffiliation{biohub}{Chan Zuckerberg Biohub, San Francisco, CA}
\icmlcorrespondingauthor{Aditi Raghunathan}{aditir@stanford.edu}
\icmlcorrespondingauthor{Gregory Valiant}{valiant@stanford.edu}
\icmlcorrespondingauthor{James Zou}{jamesz@stanford.edu}

\icmlkeywords{keywords}

\vskip 0.3in
]



\printAffiliationsAndNotice{}  

\begin{abstract} 
Given samples from a distribution, how many new elements should we expect to find if we continue sampling this distribution? This is an important and actively studied problem, with many applications ranging from unseen species estimation to genomics. We generalize this extrapolation and related unseen estimation problems to the multiple population setting, where population $j$ has an unknown distribution $D_j$ from which we observe $n_j$ samples. We derive an optimal estimator for the total number of elements we expect to find among new samples across the populations. Surprisingly, we prove that our estimator's accuracy is independent of the number of populations. We also develop an efficient optimization algorithm to solve the more general problem of estimating multi-population frequency distributions. We validate our methods and theory through extensive experiments. Finally, on a real dataset of human genomes across multiple ancestries, we demonstrate how our approach for unseen estimation can enable cohort designs that can discover interesting mutations with greater efficiency.
\end{abstract} 

\section{Introduction}
\label{sec:intro}

Given samples from a distribution, many settings in machine learning and statistics involves estimating properties of the \textit{unseen} portion of the distribution, i.e. elements in the support of the distribution that are not observed in the samples collected so far. One important example of estimating the unseen is the  problem of predicting the number of distinct new elements in additional samples collected. This question is famously illustrated by the case of Corbet's butterflies. Alexander Corbet was a British naturalist who spent two years in Malaya trapping butterflies. He found 118 rare species of butterflies for which he found only one specimen, another 74 species with two specimens, 44 with three specimens, etc. Corbet was naturally interested in the butterflies that are heretofore unseen. In particular, he wanted to estimate how many distinct \textit{new} species of butterflies he can expect to discover if he were to conduct a new expedition to Malaya---such an estimate could help determine whether a new experiment is warranted. Good-Toulmin, extending earlier work of Ronald Fisher, came up with the remarkable estimate that the number of new species Corbet can expect to find is simply the alternating sum 118 - 74 + 44 - ... The Good-Toulmin estimator sparked the investigation into how to estimate the discovery rate of new elements and this remains an active area of research. Estimating the discovery rate has many important applications beyond the original species collection setting. In genomics, for example, an important question is: given the genetic variation already identified in the genomes of individuals from some population (say, East Asia), how many additional mutations do we expect to find by sequencing the genomes of additional individuals from East Asia. An accurate answer to this question can improve the cohort design of new population sequencing experiments.       

Predicting the number of new elements is a particular instance of estimating the unseen. In other applications, one may want to estimate different statistics that also depend on the currently unobserved elements. For example, one may want to predict how many new elements will be observed at least twice (for reproducibility) or at most three times (if the focus is on rare elements). More generally, one may want to estimate the histogram of the underlying distribution, which summarizes the frequency distribution of all the elements (see Sec.~\ref{sec:related} for precise definition) and from which these other statistics can be derived.  

The unseen estimation literature has focused on the setting where there is a \emph{single} distribution which generate current samples as well as any future samples. In practice, we often have \emph{multiple} distinct distributions and we observe varying number of samples from each distribution. In the genomics example above, in addition to sequencing data from East Asians, we also have genome sequences of individuals from Europe, Africa, etc. The relevant question is: given we currently have the genomes of $n_i$ individuals from population $i$, $i\in \{1,...,m\}$, and we have identified all the genetic variants in this group, how many total new mutations do we expect to find if we sequence additional $b_i$ individuals from population $i$. Moreover, given a finite budget $N_{new}$ of new genomes that we can sequence, how should we allocate this budget across the different populations to maximize the expected number of new mutations oberved? Similarly, suppose Corbet had also collected butterflies in Brunei and Indonesia, in addition to Malaya. Then he might want to know how many totally new species he can expect to find if he was to spend, say, another six month in Malaya and one year in Brunei. He might also be interested in estimating the joint frequency distribution of butterflies across all three regions. 

\paragraph{Our contributions.} In this paper, we address the general problem of estimating the unseen when we have samples from multiple populations, each corresponding to a potentially distinct distribution. Despite being very natural, this multi-population problem has not been systematically studied to the best of our knowledge.  We derive a multi-population generalization of the Good-Toulmin estimator for the expected number of new elements. Surprisingly, we prove that the accuracy of our extrapolation estimator is independent of the number of populations. Moreover, it achieves the optimal super-linear extrapolation rate. Next, we develop an efficient optimization method to estimate the more general multi-population joint frequency distribution. This complements our extrapolation estimator, and outperforms the generalized Good-Toulmin estimator in most settings. This more general approach also enables predictions for other statistics of interest.  
We systematically validate these two algorithms on synthetic data as well as real datasets from population genetics and from English books. Moreover, we illustrate that by estimating the joint frequency distribution, we can significantly improve the discovery power under a budget constraint.

\section{Related works}
\label{sec:related}

The problem of estimating the properties of the unobserved portion of a distribution, given $n$ samples, and the related problem of estimating the number of new domain elements that are likely to be observed if an additional $c n$ samples are collected, dates back to works of I.J. Good and A. Turing~\cite{good1953population}, and R.A. Fisher~\cite{fisher1943relation}.  This was quickly followed by~\cite{good1956number}, which introduced the Good-Toulmin estimator. While the Good-Toulmin estimator is always unbiased, the variance increases rapidly for $c \ge 1$. Subsequent works, including~\cite{efron1976estimating} have suggested ``smoothing'' approaches that tradeoff the bias and variance for this type of approach. The recent work of~\citet{orlitsky2016optimal} describes a clever variant that achieves good performance for $c=O(\log n)$. This ability to accurately estimate the number of domain elements seen in a second sample of size up to $O(n \log n),$ where $n$ denotes the size of the original sample, was concurrently shown via a different approach in~\cite{valiant2016instance}. This logarithmic factor extrapolation matches the lower bounds of~\cite{valiant2011estimating}, to constant factors.
The linear estimators that we propose in Section~\ref{sec:linear_est} for the multiple population setting, and their analysis, are extensions of the smoothed Good-Toulmin estimators of~\cite{orlitsky2016optimal}.

A different approach to this problem was proposed by~\citet{efron1976estimating}, who considered a linear-programming approach to estimating this property by implicitly finding a label-less representation  of the underlying distribution that was consistent with the observed frequency counts, then returning the support size of this distribution. This approach was adapted and rigorously analyzed in~\cite{valiant2011estimating,valiant2013estimating}, who showed that it provably yields an accurate representation of the frequency distribution of the underlying distribution, which can subsequently be leveraged to yield estimates of distributional properties, including entropy, distance metrics between distributions, and approximations for the number of new elements that would be observed in larger samples. Recent works~\cite{valiant2016instance,zou2016quantifying} also established that this approach can accurately estimate the number of new elements that will be observed in samples of size up to $O(n \log n)$.   Our optimization-based algorithm, described in Section~\ref{sec:bbopt}, generalizes this approach.

\section{Definitions and examples}
\label{sec:defn}

Let $\Omega$ denote the domain, and $D_1, ..., D_m$ denote $m$ probability distributions over $\Omega$. $D_i$ represents the frequency of elements in population $i$. Note that it is not restrictive to assume that the populations share the same domain $\Omega$ since different $D_i$'s may have distinct supports. We model the multi-population unseen estimation as a two stage process. In the first period, we observe $n_j$ independent samples from the $j$-th population, $\{X_i^j\}_{i = 1,...,n_j}^{j = 1,..., m}$. This is the \textit{seen} data. In period two, which is in the future, we will sample additional $t_j n_j$ samples from the $j$-th population, $\{Y_i^j\}_{i = 1, ..., t_j n_j}^{j =1 ,..., m}$. The period two samples are \textit{unseen} and we would like to estimate some statistic $U(\{Y_i^j\}, \{X_i^j\})$. We can think of $t_j \geq 0$ as the extrapolation factors. If $t_j$ is large, then we will obtain many more samples from population $j$ in the second period compared to what we have, and the problem of estimating $U$ could be more challenging. We can take $t_j$ as given for the purpose of estimating $U$. We later discuss how we to leverage our estimator of $U$ to optimize the $t_j$'s in order to maximize the number of new discoveries. Note that in general, the $n_j$'s and $t_j$'s can differ arbitrarily across the populations.  

A particularly important statistic is $U= $  \emph{the total number of new elements in $\{Y_i^j\}$ that are not observed in the period one samples $\{X_i^j\}$}. A good estimator for this $U$ quantifies the expected information gain of the second period. In the one population setting, this statistic is the focus of Good-Toulmin and a large number of papers. Other useful choices of $U$ could be the number of distinct new elements that are observed at least twice in $\{Y_i^j\}$, which could be relevant if we want some reproducibility.  

Beyond estimating these single parameters, we could also hope to use the samples $\{X_i^j\}$ to estimate the \textit{histogram} of $D_1, ..., D_m$.  The multi-population \emph{histogram}, defined below, captures all of the information about the populations, other than the labels of the domain.

\begin{definition}[Multi-population histogram]
Given a collection of $m$ distributions $D_1,\ldots,D_m$ over a common domain $\Omega$, the corresponding multi-population histogram  $\sc{H}$ is a mapping from $[0, 1]^m\setminus 0^m\ \mapsto \mathbb{N} \cup \{0\}$. For each $\boldsymbol{\alpha} = (\alpha_1, \alpha_2, \hdots \alpha_m) \in [0, 1]^m\setminus 0^m, H(\boldsymbol{\alpha}) = | \{y \in \Omega \mid D_j(y) = \alpha_j,  1 \leq j  \leq m\}|$, where $D_i(y)$ is the probability mass of domain element $y$. 
in the $i$th distribution $D_i$.
\end{definition}

Any \emph{symmetric} multi-population statistic---one that is invariant to permuting the labels of the domain---is a function of only the histogram.  Such statistics include distance metrics between the distributions/populations, measures of the entropy of the populations, and the number of new elements that one is likely to observe in a second batch of samples.  The multi-population histogram is also of intrinsic interest; in population genetics, $\sc{H}$ is exactly the joint frequency distribution of mutations, and reveals information about demographic history (e.g. historical variations in population size) and selective pressures.  One benefit of focusing on the histogram is that, while it does not contain as much information as the actual labeled distributions, it can often be accurately recovered even when given too few samples to learn the (labeled) distributions to any significant accuracy~\cite{valiant2011estimating}.

Both for directly predicting $U$ and estimating $\sc{H}$, we rely on a label-less representation of the samples, termed the \emph{fingerprint} of $\{X_i^j\}$.  The fingerprint of the samples is the analog of the histogram of the distributions, and captures all the information of $\{X_i^j\}$ that is relevant for estimating symmetric statistics. 

\begin{definition}[Multi-population fingerprint]
Given the samples $\{X_i^j\}$, its fingerprint is an $m$-dimensional tensor $\Phi$ whose $i_1...i_m$-th entry, 
$\phi_{i_1...i_m},$ is the number of distinct elements observed exactly $i_j$ times in the samples from population $j$. Here each $i_j$ can range from 0 to $n_j$.
\end{definition}

\begin{example}
Suppose we have five samples from Population 1, $(A,B,C,E,F)$, and seven from Population 2, $(A,B,D,E,E,F,F)$. 
The corresponding 2-dimensional fingerprint of this data is given by the following matrix:
\begin{center}
\begin{tabular}{ c|*{3}{c} }
& 0 & 1 & 2 \\
\hline
0 & $\cdot$ & 1 & 0 \\
1 & 1 & 2  & 2 \\
\end{tabular}
\end{center}
The $(1, 1)$ entry is 2 because $A, B$ are observed once in each set of samples; the $(1, 0)$ entry is 1 because exactly one element, $C$, is observed once in the samples from Population 1 and zero times in the samples from Population 2.  By convention, we omit the $(0,0)$ element.  
\end{example}

\section{A linear estimator}
\label{sec:linear_est}

\paragraph{Unbiased estimator.}
Given the empirical fingerprints $\Phi$ and the extrapolation factors $t_j, j=1,...,m$, we define the following estimator 
\begin{eqnarray}
\hat{U} = -\sum_{i_1, ..., i_m: \sum i_j > 0} \left(\prod_{j=1}^m (-t_j)^{i_j} \right) \phi_{i_1 ... i_m}.
\label{eqn:unbiased}
\end{eqnarray}
$\hat{U}$ is a weighted alternating sum of the empirical fingerprints where the weights are determined by the extrapolation factors $t_j$.

\begin{proposition}
For any number of populations $m$, and any extrapolation factors $t_j \geq 0, j=1,...m$, $\hat{U}$ is an unbiased estimator of $U$.
\label{prop:prop1}
\end{proposition}
Proof of the proposition appears in Appendix~\ref{app:prop1guarantee}. 

$\hat{U}$ is linear in the fingerprint entries. Its computational cost is linear in the total number of period one samples, $n=\sum_j n_j$, since there can be at most $n$ non-zero fingerprint entries. To build more intuition for $\hat{U}$, we illustrate its application in two simple settings. 


\begin{example} Consider the setting where all $m$ distribution are identical, i.e. all the samples are drawn from the same discrete distribution $D$. Let $ t_j = 1, \forall j$ for simplicity. After rearranging terms, $\hat{U}$ can be written as
\[
\hat{U} = \sum_{k = 1} (-1)^{k+1} \left(\sum_{(i_1,\ldots,i_m) : \sum i_j = k} \phi_{i_1 ... i_m} \right).
\]
Because the populations are identical, the sum in the parenthesis is just the number of elements  that are observed $k$ times from all the samples so far. Hence the general estimator $\hat{U}$ reduces to the one dimensional Good-Toulmin estimator when all $m$ populations are identical. 
\end{example}

\begin{example} Suppose the supports of the distributions $D_i$ are disjoint. Then the only possible non-zero fingerprint entries are $\phi_{i_1...i_m}$ where exactly one of the $i_j$ is great than 0 and all the other $i_j$'s are zero. For simplicity, assume $t_j = 1$ for all $j$. Then $\hat{U} = \sum_{j = 1}^k \sum_i (-1)^{k+1}\phi_i^k$, where $\phi_i^k$ is the marginal fingerprint entry of the number of elements that are observed $i$ times in population $k$. Hence when the populations are disjoint, the expected number of new elements is the sum of the expected number of new elements in each population. When the populations have overlapping support, we have the nontrivial interaction terms due to the cross-population fingerprint entries.
\end{example}

\textbf{General weighted linear estimator.} While $\hat{U}$ is unbiased, its variance could be large if some of the extrapolation factors $t_j$'s are greater than 1. This is because the powers of $t_j$ appear in Eqn.~\ref{eqn:unbiased}. To address this issue, we introduce a general class of multi-population weighted linear estimators. 
\vspace{-.2cm}\[
\hat{U}^W = -\sum_{i_j: \sum i_j > 0} \left(\prod_{j=1}^m (-t_j)^{i_j} \right) \phi_{i_1, \hdots, i_m} W(i_1, \hdots, i_m). 
\] 

\vspace{-.2cm}We focus on a particular weighting scheme, which is an extension of that introduced in~\cite{orlitsky2016optimal}: $W(i_1, i_2, \hdots i_m)= \mathbb{P}\left(L \ge \sum_{j \in A} i_j \right)$ where $L \sim \mbox{Poi}(r)$ and $A = \{j: t_j >1\}$ are the populations that we would like to extrapolate beyond the original sample size. If $t_j \leq 1 ~ \forall j$, then W = 1 and $\hat{U}^W $ is just the unbiased estimator $\hat{U}$. The Poisson rate $r$ is a tuning parameter that determines the bias/variance tradeoff of $\hat{U}^W$. As $r$ increases, all the weights approaches 1 and $\hat{U}^W$ approaches the unbiased estimator $\hat{U}$. As $r$ decreases, the fingerprint entries $\phi_{i_1...i_m}$ with some large $i_j$'s---which are also the terms with high variance---are weighted  by a factor that is close to 0. This reduces the total variance of $\hat{U}^W$ at the cost of introducing bias. We will see how to set $r$ as a function of the $n_j$'s and $t_j$'s in order to minimize the overall estimation error. In the rest of the paper, unless otherwise specified, we will use $\hat{U}^W$ to denote the multi-population linear estimator with Poisson weights. 

\vspace{-.2cm}\paragraph{Performance guarantee of the weighted estimator.} We use relative mean squared error,  $\mathbb{E}\left[ \left( \frac{\hat{U}^W - U}{ \sum n_j t_j} \right)^2 \right] $, to quantify the performance of $\hat{U}^W$. This is a natural error metric, because $\sum n_j t_j $ is the number of  samples in period two and we care about how the error in the predicted number of new elements scales with the number of samples. Without loss of generality, we can relabel the populations so that $t_1 = max_j t_j$. We are especially interested in the setting when $t_1 \geq 1$ (i.e. large extrapolation). 

\begin{proposition} Suppose $t_1 = \max_j t_j \geq 1$ and the Poisson rate is $r = \frac{\log (\sum_j n_j (t_j + 1))}{2t_1}$, then 
\begin{eqnarray}
\mathbb{E}\left[ \left( \frac{\hat{U}^W - U}{ \sum n_j t_j} \right)^2 \right] \leq \left( \frac{ n_1t_1+ \sum_j n_j }{n_1 t_1} \right) n_1^{-1/t_1}.
\label{eqn:guarantee}
\end{eqnarray}
\label{prop:guarantee}
\end{proposition}

\begin{remark}[$\log$ extrapolation factor]
 Suppose the ratio $\frac{n_1}{\sum_j n_j}$ is bounded, then Prop.~\ref{prop:guarantee} guarantees that for any $\epsilon > 0$, we can achieve $\mathbb{E}\left[ \left( \frac{\hat{U}^W - U}{ \sum n_j t_j} \right)^2 \right] \leq \epsilon$ with $t_1 = O(\log n_1/\log(1/\epsilon))$. This means that $\hat{U}^W$ has low relative error even when the largest extrapolation factor $t_1$ is logarithmic in its initial sample size $n_1$. 
 \end{remark}

\begin{remark}[no dependence on $m$]
 Note that the relative error in Eqn.~\ref{eqn:guarantee} \textbf{does not depend on the number of populations $m$}. This is somewhat surprising since the number of terms in $\hat{U}^W$ potentially grows exponentially with $m$ and the variance of each fingerprint entry $\phi_{i_1...i_m}$ also increases as the number of population increases. This population agnostic property of $\hat{U}^W$ guarantees its accuracy even when $m$ is arbitrarily large. 
 \end{remark}

\begin{remark}[lower bound]
Here we have focused on a specific form of the estimator $\hat{U}^W$ where the weights $W$ of the fingerprint entries correspond to the tail probability of Poisson distributions. A natural question is whether there exists a different form of the weights or a different estimator altogether that can consistently be more accurate than our current $\hat{U}^W$. The answer is essentially \textit{no} due to the following lower bound for one population extrapolation~\cite{orlitsky2016optimal,valiant2011estimating}: There exists universal constants $c,c'$ such that for all estimators $\hat{U}$, if the extrapolation factor $t>c$, then $\exists$ distribution such that $\mathbb{E}\left[ \left( \frac{\hat{U} - U}{ n t} \right)^2 \right] \underset{\sim}{>} n^{-c'/t}$. Here $n$ is the number of samples drawn from this distribution in period one. This lower bound implies that in order to guarantee that the relative error is less than $\epsilon$ in general, the extrapolation factor can be at most $O(\log n/\log(1/\epsilon))$, matching Prop.~\ref{prop:guarantee}. 
\end{remark}


\textit{Outline of the proof of Prop.~\ref{prop:guarantee}} (detailed analysis is in the Appendix). To analyze the relative error, we separately quantify the bias and variance of $\hat{U}^W$ in terms of $n_j, t_j$, $r$. 

\begin{lemma}[Bias] Let $r$ denote the rate of the Poisson weights, then
\[
\left| \mathbb{E}[\hat{U}^W - U] \right| \leq \left(\sum_{j\in A} n_j (t_j + 1) \right)e^{-r}
\]
\label{lem:bias}
\end{lemma} 
\begin{lemma}[Variance]
Without loss of generality, let $t_1 = \max_j t_j$ and suppose $t_1 \geq 1$ then
\[
\mbox{Var}(\hat{U}^W - U) \leq \sum n_j  e^{2r(t_1-1)} + \sum_j n_j t_j.
\]
\label{lem:var}
\end{lemma} 

\vspace{-0.5cm}To obtain the optimal $r$ given in the statement of Prop.~\ref{prop:guarantee}, we set $r$ to balance the squared bias and variance. 

\section{Estimating the multi-population frequency distribution}\label{sec:bbopt}

While we have a linear estimator for the number of unseen elements in a new sample, it is challenging to construct good estimators of other statistics (e.g. number of new elements observed $\geq 2$) directly from the fingerprints. As discussed in Sec.~\ref{sec:defn}, we can also take the less direct approach of first trying to estimating the true underlying multi-population histogram. Given an accurate reconstruction of this underlying histogram, we can then estimate any symmetric statistic of the future samples.  
We discuss some of the uses of such a representation in Section~\ref{sec:otherUses}.  

\paragraph{Recovering the frequency distribution}
The core of our algorithm to recover
the multi-population histogram is a natural extension of the single population algorithm presented in \citet{valiant2011estimating,valiant2013estimating}.

\begin{mdframed}
\textbf{Estimating the multi-population histogram: Core Approach.}\\
\textbf{Input:} Multi-population fingerprint $\Phi$ of samples, \\
\textbf{Output:} Two estimates, $\hat{\sc{H}}_{counts}$ and $\hat{\sc{H}}_{ll}$ of histogram corresponding to the  distributions underlying fingerprint $\Phi$.
\begin{itemize}
\item Compute $\hat{\sc{H}}_{counts}$ and $\hat{\sc{H}}_{ll}$ minimizing the following expressions:
$$ \hspace{-0.1cm}\hat{\sc{H}}_{counts} = \arg\min_H\sum\limits_{\boldsymbol{i}} \frac{1}{\sqrt{1+\Phi_{\boldsymbol{i}}}} 
| \Phi_{\boldsymbol{i}} - \hat{\Phi}(H)_{\boldsymbol{i}}|.$$
$$\hat{\sc{H}}_{ll} = \arg\max_H \sum\limits_{\boldsymbol{i}} \log poi(\Phi_{\boldsymbol{i}}, \hat{\Phi}(H)_{\boldsymbol{i}}),$$
$$\text{Where  } [\hat{\Phi}(\sc{H})]_{\boldsymbol{i}} = \sum\limits_{\boldsymbol{\alpha}} H(\alpha) \prod\limits_{j=1}^m bino(\alpha_j, n_j, i_j).$$
\end{itemize}
\end{mdframed}

The intuition behind these two optimization problems is the following.  The histogram corresponding to a set of distributions is an \emph{unlabeled} representation of the underlying distributions, hence it makes intuitive sense to try to recover the histogram that maximizes the likelihood of the \emph{unlabeled} representation of the samples, namely the fingerprint $\Phi.$ Recent work~\citep{acharya2016unified} provided rigorous support for this intuition.  In general, however, this likelihood might be difficulty to compute.  Nevertheless, an efficiently computable proxy for this likelihood can be obtained by treating the distribution of the fingerprint, corresponding to a histogram $H$, as a product distribution, with $\Phi_{i_1,\ldots,i_m}$ distributed according to the Poisson distribution with appropriate expectation $E_H[\Phi_{i_1,\ldots,i_m}].$   The recent central limit theorem for ``Poisson Multinomials'' from~\cite{valiant2011estimating} provides at least some corroboration for the reasonableness of having a proxy for the log-likelihood that decomposes linearly across the different elements of $\Phi.$    The motivation for the $\frac{1}{\sqrt{1+\Phi_{\boldsymbol{i}}}}$ scaling on the first proxy likelihood function is that this expression  penalizes discrepancies between the observed and expected fingerprint entries according to a rough approximation of the standard deviation of that fingerprint entry, as the variance of a Poisson random variable is equal to its expectation, and the observed fingerprint entry is an approximation for the expected fingerprint entry given the true underlying histogram.

The work~\cite{valiant2013estimating} focused on recovering $\hat{\sc{H}}_{counts}$, as this optimization problem can be formulated as a linear program, whose variables correspond to a fine discretization of the potential support of the histogram.   Unfortunately, in the present multi-distribution setting, the number of variables required by this linear programming approach would scale exponentially with the number of distributions in question.  Even for fingerprints derived from modest-sized samples from two distributions, the resulting linear program becomes impractical.

Instead of pursuing the linear programming based approach, we instead propose a black-box optimization approach to finding a histogram that optimizes either of the two proxy likelihood functions.  In this optimization approach, the dimensionality of the optimization problem is specified by the user, and corresponds to the number of $(i_1,\ldots,i_m)$ tuples for which the returned histogram $\hat{\sc{H}}$ is nonzero.   Denoting this quantity by $s$, the resulting optimization problem can be regarded as the problem of specifying $s$ vectors $(h_1,\alpha_{1,1},\ldots,\alpha_{1_,m}),\ldots,(h_s,\alpha_{s,1},\ldots,\alpha_{s,m})$.  These $s$ vectors are then interpreted as a histogram $H$ with $H(\boldsymbol{\alpha}_j)=h_j$ for all $j \in \{1,\ldots,s\}$, and $H(\boldsymbol{\alpha}) = 0$ for all other vectors $\boldsymbol{\alpha}$.

The one additional modification that leads to a substantial improvement in runtime is to only evaluate the proxy likelihood expressions for fingerprint entries $\Phi_{i_1,\ldots,i_m} \ge 2$.  The intuition for this is two-fold.  First, the number of vectors $(i_1,\ldots,i_m)$ for which $\Phi_{i_1,\ldots,i_m} = 0$ will scale exponentially with $m$, as opposed to scaling as some parameter of the sample sizes; this is clearly undesirable.  Second, given that we wish to avoid evaluating the contribution to the proxy likelihood from fingerprint entries that are zero, we must now be careful in dealing with fingerprint entries that are equal to 1. Suppose we have 1 element with true probability $\frac{i}{n}$ and suppose we observe that fingerprint entry $\Phi_i = 1$, and the other fingerprints near $i$ are $0$. Since we are maximizing the likelihood that $\Phi_i = 1$ (without taking into account the nearby 0 entries), we would assign roughly $\sqrt{i}$ elements to probability $\frac{i}{n}$ which is undesirable. Removing the ones largely resolves this issue. Note that the $\Phi_j=2$ entries do not cause as much of an issue, as such collisions are unlikely to occur in regions of the fingerprint in which there is not a significant number of domain elements. 

In this one-distribution example, a constraint on the total probability mass being 1 would resolve this issue, though analogs of this issue in the multiple distribution setting cannot be resolved in this way.  Hence, we adopt the crude, but effective approach of viewing all the empirical fingerprint entries that are equal to 1 as being reflective of an element in the underlying set of distributions whose probability is close to the empirical probability of the corresponding element.  We summarize the complete algorithm below:

\begin{framed}
\textbf{Estimating the multi-population histogram: Full Algorithm.}\\
\textbf{Input:} Multi-population fingerprint $\Phi$ derived from samples from $m$ distributions of respective sizes $n_1,\ldots,n_m$. \\
\textbf{Output:} Two estimates, $\hat{\sc{H}}_{counts}$ and $\hat{\sc{H}}_{ll}$ of histogram corresponding to the  distributions underlying fingerprint $\Phi$.
\begin{itemize}[leftmargin=*]
\vspace{-.2cm}\item Remove fingerprint entries that are 1, and add to empirical portion of histogram:
\begin{enumerate}
\vspace{-.2cm}\item Initialize $m$-distribution histogram $\hat{\sc{H}}_{emp}$ to be identically zero.
\item For each vector $\boldsymbol{i}=(i_1,\ldots,i_m)$ such that $\Phi_{\boldsymbol{i}}=1$, set $\hat{\sc{H}}_{emp}(\frac{i_1}{n_1},\ldots,\frac{i_m}{n_m}) = 1.$
\end{enumerate}
\vspace{-.2cm}\item Compute $\hat{\sc{H}}_{counts}$ and $\hat{\sc{H}}_{ll}$ minimizing the following expressions:
\begin{align*}
&\hat{\sc{H}}_{counts} = \arg\min_H \sum\limits_{\boldsymbol{i}:\Phi(\boldsymbol{i}) \ge 2} \frac{1}{\sqrt{1+\Phi_{\boldsymbol{i}}}} | \Phi_{\boldsymbol{i}} - \hat{\Phi}(H)_{\boldsymbol{i}}|.\\
& \hat{\sc{H}}_{ll} = \arg\max_H \sum\limits_{\boldsymbol{i}:\Phi(\boldsymbol{i}) \ge 2} \log poi(\Phi_{\boldsymbol{i}}, \hat{\Phi}(H)_{\boldsymbol{i}}.\\
&\text{Where } \hat{\Phi}(\sc{H})_{\boldsymbol{i}} = \sum\limits_{\boldsymbol{\alpha}} H(\alpha) \prod\limits_{j=1}^m bino(\alpha_j, n_j, i_j).
\end{align*}
Subject to the constraint that, together with $\hat{\sc{H}}_{emp},$ the total mass in all the distributions is 1. Namely for all $i \in \{1,\ldots, m\},$ $$\sum_{\boldsymbol{\alpha}}\alpha_i\hat{\sc{H}}_{ll}(\boldsymbol{\alpha})  +\sum_{\boldsymbol{\alpha}} \alpha_i \hat{\sc{H}}_{*}(\boldsymbol{\alpha})= 1.$$
\vspace{-.6cm}\item Return the concatenation of the empirical portion of the histogram and the portion returned by the optimization: $\hat{\sc{H}}_{count}:=\hat{\sc{H}}_{count}+ \hat{\sc{H}}_{emp}$, and $\hat{\sc{H}}_{ll}:=\hat{\sc{H}}_{ll}+ \hat{\sc{H}}_{emp}$
\end{itemize}
\end{framed}

\paragraph{Leveraging $\hat{\sc{H}}$ for approximating the value of additional data.}\label{sec:otherUses}
An accurate representation of the histogram corresponding to the multi-population distribution underlying a given set of observations can be leveraged to estimate a number of useful properties.  These properties include estimating the number of new domain elements that would likely be seen given additional samples from the populations.  Specifically, given a histogram $\hat{H}$, corresponding to $m$ populations, we can estimate the expected number of distinct elements that will be observed in samples from the $m$ populations of respective sizes $n_1,\ldots,n_m$ via the simple formula:
\begin{eqnarray}
E[\text{num observed}] = \sum_{\boldsymbol{\alpha}}\hat{H}(\boldsymbol{\alpha})\left(1-\prod_{i=1}^m(1-\alpha_i)^{n_i}\right).
\label{eqn:extrapolation}
\end{eqnarray}

\vspace{-.2cm}An accurate approximation to the histogram can also be leveraged to answer many other questions about the populations that can not be readily addressed via the linear estimators of Section~\ref{sec:linear_est}.  These include tasks such as estimating the amount of data that must be collected to capture, say, $99\%$ of the mass of the distributions in question.

\begin{figure*}[!tb]
\centering
\includegraphics[ scale = 0.22, trim={1.5cm 3cm 1.5cm  0}]{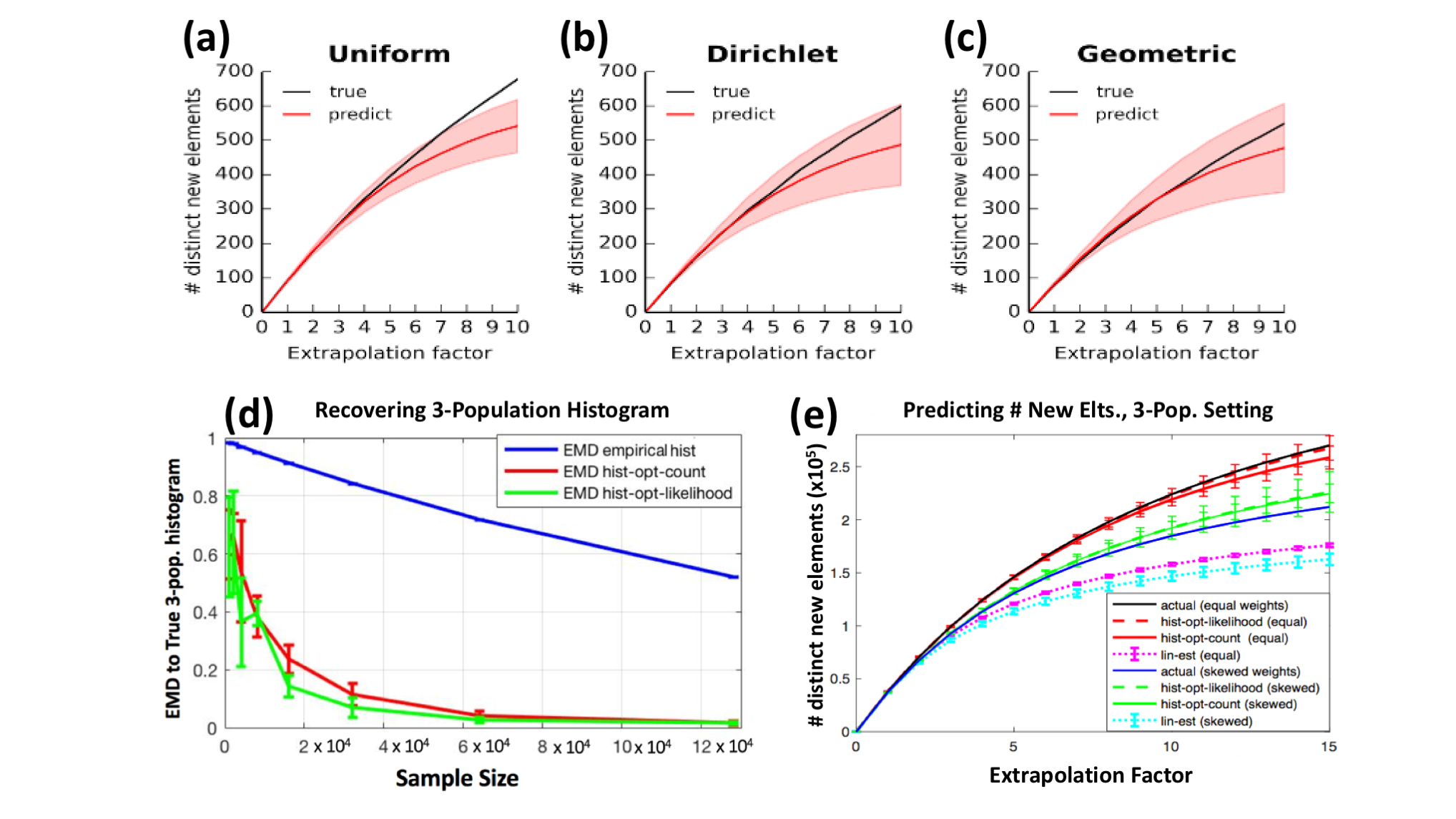}
\caption{Performance of the weighted linear estimator of Prop.~\ref{prop:guarantee} for \textbf{(a)} Uniform, \textbf{(b)} Dirichlet and \textbf{(c)} Geometric distributions. Each experiment contains 100 populations. The $x$-axis corresponds to the maximum extrapolation factor among the 100 populations ($t_1$ in Prop.~\ref{prop:guarantee}). The black curve indicates the true number of distinct new elements that we expect to observe in the new samples, and the red curve shows the predicted number of new elements. The red shaded region corresponds to one standard deviation over 100 independent experiments.  \textbf{(d)} The 3-population earthmover distance (EMD) between the recovered histograms and the true histogram corresponding to the populations from which the samples were drawn.  The blue line corresponds to the histogram of the empirical distribution of the samples, and the red and green lines correspond to the histograms returned by our multi-population histogram estimation algorithm, using the count-objective and likelihood objectives, respectively.  Plots depict the mean and standard deviation over 5 independent runs.  The true underlying distribution is supported on $4\cdot 10^5$ domain elements.  \textbf{(e)}  Estimating the number of new domain elements that will be observed given additional samples in the same 3-population setting. Estimates are made for when the new samples are evenly distributed among the populations (equal) and when the majority come from one population (skewed). Error bars depict one standard deviation about the mean, calculated based on 10 independent trials. } 
\label{fig:combined}
\end{figure*}

\begin{figure*}[!tbp]
  \centering
  \includegraphics[scale = 0.4, trim={1cm 1.5cm 1cm .7cm}]{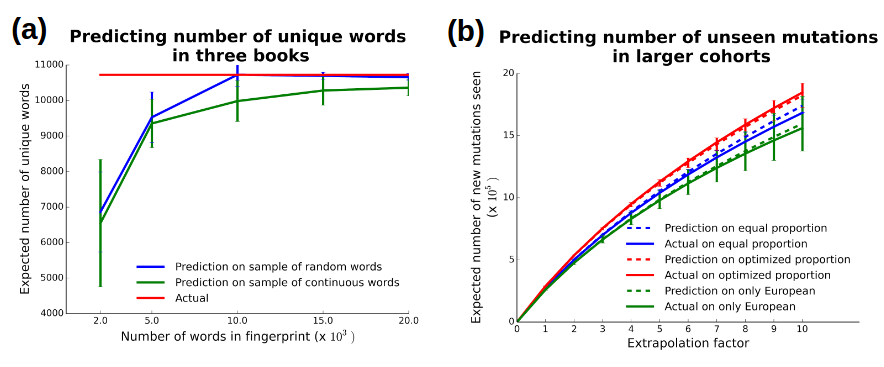}
  \caption{  \textbf{(a)} Estimating the total number of unique words combining three different books using hist-opt-counts. Predictions are based on fingerprints of samples of words---sampled without replacement either randomly (blue) or from a contiguous block of text (green) from each book. Error bars depict one standard deviation over 10 independent runs. \textbf{(b)} Estimating the number of new mutations that would be observed given additional samples from four different populations using hist-opt-counts. We consider different ratios of sampling within these  populations and observe the change in number of new mutations that would be observed.\label{fig:f2}}
\end{figure*}

\vspace{-0.2cm}

\section{Experiments}

\paragraph{Evaluating the weighted linear estimator for large $m$.} We empirically evaluated the performance of the weighted linear estimator $\hat{U}^W$. The experiments were conducted for three types of distributions---Uniform, Dirichlet and Geometric---that are commonly used to evaluate extrapolation algorithms. Each experiment contains $m=100$ populations. We have a total of 3000 distinct elements. In the Uniform setting, each population has support on 100 elements that are randomly sampled from the 3000. For Dirichlet, each population also has support on 100 random elements (from the 3000), and the weights on these 100 elements are sampled from a Dirichlet prior. For the Geometric experiments, each population corresponds to a random ordering of the 3000 elements and the $k$-th element is assigned probability  $\propto  (1-p)^k p$. In period one, ten samples are observed in each of the 100 populations. In period two, 95 randomly chosen populations have extrapolation factor $t \in [0,1]$ and five populations have extrapolation factor $10t$. This simulates the setting where we can obtain substantially more samples from a subset of the populations. 

Figure~\ref{fig:combined}(a, b, c) shows the results of the experiments for Uniform, Dirichlet(1) and Geometric with $p = 0.05$ respectively. The results for other parameter settings are qualitatively similar. The black curves indicate the true number of distinct new elements we expect to observe in period two by sampling from the true underlying distributions. The red curves are the predictions of the weighted linear estimator (shaded regions indicate one standard deviation across 100 experiments). In all three settings, $\hat{U}^W$ provides accurate estimate with low variance when the maximum extrapolation factor is relatively small ($\le 3$). For Uniform and Geometric distributions, the accuracy is high up to 10 fold extrapolation. For Zipf, the bias is low but variance becomes large for the maximum extrapolation factor around 10. The downward bias in the predictions is due to the weighting scheme. The relative error of the weighted estimator, $\left(\frac{\hat{U}^W - U}{\sum n_j t_j} \right)^2$, is 0.09, 0.08 and 0.08 for the Uniform, Dirichlet and Geometric distributions when the maximum extrapolation factor is 10. This confirms the theoretical results of Prop.~\ref{prop:guarantee} on the accuracy of the weighted linear estimator. 

\textbf{Evaluating the histogram estimators.}
We first validated the performance of $\hat{H}_{count}$ and $\hat{H}_{ll}$ on a three population setting with synthetic data. The true population consists of three uniform distributions over 200k elements, whose supports have 100k elements in common, and 100k elements unique to each distribution. In Figure~\ref{fig:combined}(d), the x-axis corresponds to the number of samples we observe from each population, and the y-axis indicates the earthmover distance (EMD) between $\hat{H}_{count}$, $\hat{H}_{ll}$ and the true histogram. As a baseline, we also compute the EMD between the empirical histogram of the observed samples and the true histogram. $\hat{H}_{count}$ and $\hat{H}_{ll}$ performed roughly equally well and both are substantially better than the empirical estimator especially when the number of observed samples is small. Figure~\ref{fig:combined}(e) illustrates the extrapolation accuracy of our histogram estimators. We estimated $\hat{H}_{count}$ and $\hat{H}_{ll}$ using 16K from each population, and then used Eqn.~\ref{eqn:extrapolation} to estimate the number of unseen elements in additional samples. We tested two different settings: 1) when the additional samples are equally drawn from the three populations, and 2) a skewed mixture where 5/6 of the new samples are from population 1 and 1/12 each are drawn from population 2 and 3.  $\hat{H}_{count}$ and $\hat{H}_{ll}$ gave extremely accurate predictions. In comparison, the weighted linear estimator $\hat{U}^W$ was accurate for the initial extrapolations but has downward bias when the extrapolation increases, consistent with Fig.~\ref{fig:combined}(a-c).  

Additionally, we evaluate the performance of $\hat{H}_{count}$ on a real dataset, in which we sampled words from three books--\textit{Hamlet} (32K total words), \textit{Treasure Island} (40K) and \textit{The Sun Also Rises} (72K). We used the true word frequencies (over the entire text) as the true histogram. We sampled a small number of words (equal in all books) either randomly or from a contiguous block of text and used $\hat{H}_{counts}$ to predict the total number of distinct words in total in all three books. In Figure~\ref{fig:f2}(a), the red line is the true value, and blue and green lines are predictions based on $\hat{H}_{count}$ derived from samples of either random words, or words occurring in a random contiguous block of text, respectively.  We obtain accurate estimates using a fraction of words (10K from each book). The estimates based on independent samples of words is more accurate than that based on contiguous blocks of text---likely due to correlation in words that occur near each other.  

\textbf{Optimizing discovery rate.} Given the estimated histogram $\hat{H}_{count}$ or $\hat{H}_{ll}$, we can optimize the allocation of new samples across the populations to maximize the number of unseen elements we can expect to discover given a bound on $\sum_j t_j n_j$. To illustrate, we obtained genome sequencing data of 45K individuals from the Exome Aggregation Consortium \cite{lek2016analysis}. The individuals come from four ancestries: Europeans, Africans, East Asians and Latinos. We used all the observed mutations from the 45K samples to construct a four population frequency distribution. For the experiment, we treat this as the ground truth and sampled $10^5$ mutations from each population to obtain ``seen'' data. Suppose we have budget to sample $3 \times 10^6$ variants (10 fold extrapolation from current sample size), how should we allocate these new samples across the four populations in order to maximize the number of new variants discovered? We use $\hat{H}_{count}$ to predict the extrapolation curves for three scenarios: 1. all the samples are allocated to Europeans (current genomic studies are heavily enriched of Europeans); 2. the samples are evenly allocated across the four populations; 3) we explicitly optimize the factors $t_j$ using $\hat{H}_{count}$. The dotted curves in Fig.~\ref{fig:f2}(b) correspond to the predictions, and the solid curves are the actual numbers using the true distribution, showing good agreement. Optimization using $\hat{H}_{count}$ led to 10 \% increase in the number of new variants discovered. 
This is a simplistic example (there are many other factors in the design of real cohorts) but it illustrate the potential power in having multi-population histogram estimates. In Appendix Fig.~\ref{fig:4pop_words2}, we also show that $\hat{H}_{count}$ gives accurate predictions for a different statistic---the  number of new variants we expect to find at least twice in the new samples.

\vspace{-0.1cm}
\section{Discussion}
\vspace{-0.1cm}
We introduce and formalize the problem of multi-population unseen estimation. We provide a weighted linear estimator for the number of new elements and a general optimization algorithm to estimate the multi-population histogram. These two approaches have complementary strength. The weighted linear estimator $\hat{U}^W$ specifically estimates the number of unseen elements. It's accuracy is independent of the number of populations, $m$, and it is worst-case optimal. This can be a good method especially when $m$ is large and the extrapolation factor is small compared to $\log$ of the number of observed samples. When the extrapolation is larger, however, $\hat{U}^W$ is consistently downward biased due to its variance-reducing weights.  For relatively small number of populations ($m = 2,3,4$) and larger extrapolation factors, the unseen predictions of our histogram estimators, $\hat{H}_{count}$ and $\hat{H}_{ll}$ are significantly more accurate than $\hat{U}^W$.   While both likely have comparable worst-case performance, the linear estimator nearly always incurs this worst-case loss and is largely incapable of extrapolating beyond this worst-case logarithmic factor. In contrast, the histogram-based estimators seem to perform well for much larger extrapolation factors on all of the distributions that we considered. $\hat{H}_{count}$ and $\hat{H}_{ll}$ are computationally more expensive than $\hat{U}^W$, but are still tractable for many applications---each run of our experiments took less than 20 minutes on a single laptop.

\newpage
\section*{Acknowledgments}  Gregory Valiant's contributions were supported by NSF CAREER CCF-1351108 and a Sloan Research Fellowship. James Zou is a Chan Zuckerberg Biohub investigator and is supported by NSF CISE-1657155.
 
\bibliography{multipop_unseen}
\bibliographystyle{icml2017}

\newpage
\onecolumn
\appendix

\label{app:prop1guarantee}
\section{Proof of Prop.~\ref{prop:prop1} and Prop.~\ref{prop:guarantee}}

\begin{proof}[Proof of Prop.~\ref{prop:prop1}]

For each element $x \in \mathcal{X}$, let $\lambda_{x, j} = n_j p_{x,j}$, where $p_{x, j}$ is the probability of $x$ in population $j$. We have
\[
\mathbb{E}[U] = \sum_x e^{-\sum_j \lambda_{x,j}}\left( 1 - e^{-\sum_j t_j \lambda_{x,j}} \right).
\]
The first term in the sum is the probability that $x$ is not observed in period one and the second term is the probability that $x$ is observed at least once in period two. Taylor expand the second term followed by Binomial expansion gives
\begin{eqnarray*}
\mathbb{E}[U] &=& \sum_x e^{-\sum_j \lambda_{x,j}}\sum_{i = 1}^{\infty} (-1)^{i+1}\frac{(\sum_j t_j \lambda_{x,j})^i} {i!} \\
&=& -\sum_x e^{-\sum_j \lambda_{x,j}} \sum_{i_1, ..., i_m: \sum i_j > 0} \prod_{j=1}^m \frac{(-t_j \lambda_{x,j})^{i_j}}{i_j !} \\
&=& -\sum_{i_1, ..., i_m: \sum i_j > 0} \sum_x e^{-\sum_j \lambda_{x,j}}\prod_{j=1}^m \frac{(-t_j \lambda_{x,j})^{i_j}}{i_j !}\\
&=& -\sum_{i_1, ..., i_m: \sum i_j > 0} \left(\prod_{j=1}^m (-t_j)^{i_j} \right) \mathbb{E}[\phi_{i_1 ... i_m}].
\end{eqnarray*}

It's easy to see that $\hat{U}$ is an unbiased estimator of the last expression. 
\end{proof}

Weighting the fingerprints reduces the variance of the estimator at the cost of introducing bias. We analyze the bias and variance of $\hat{U}^W$ separately. The proof follows the strategy of the analysis for the one population setting in~\cite{orlitsky2016optimal}.

\begin{lemma}[Lemma~\ref{lem:bias} restated] Let $n = \sum_{j=1}^m n_j$ denote the total number of samples in period one and $r$ denote the rate of the Poisson weights, then
\[
\left| \mathbb{E}[\hat{U}^W - U] \right| \leq \left(\sum_{j\in A} n_j (t_j + 1) \right)e^{-r}
\]
\end{lemma}
\begin{proof}
For each element $x$, its contribution to  $\mathbb{E}[\hat{U}^W]$ can be written as 
\begin{eqnarray*}
& &-e^{-\sum_{j=1}^m \lambda_{x, j}}\left(\sum_{i_1,...,i_m}\prod_{j=1}^m (-t_j)^{i_j}\frac{\lambda_{x,j}^{i_j}}{i_j!}\mathbb{P}\left(L\geq \sum_{j\in A}i_j \right) - 1 \right)\\
&=& -e^{-\sum_{j=1}^m \lambda_{x, j}}\left(\left[\sum_{i_j: j\not\in A}\prod_{j \not\in A} (-t_j)^{i_j}\frac{\lambda_{x,j}^{i_j}}{i_j!} \right] \left[ \sum_{i_j: j\in A}\prod_{j \in A} (-t_j)^{i_j}\frac{\lambda_{x,j}^{i_j}}{i_j!}\mathbb{P}\left(L\geq \sum_{j\in A}i_j \right)\right]-1\right) \\
&=& -e^{-\sum_{j=1}^m \lambda_{x, j}}\left( e^{-\sum_{j\not\in A} t_j\lambda_{x, j}}\sum_{i=0}^{\infty} \frac{\mathbb{P}(L \geq i)}{i!} \left( -\sum_{j\in A} t_j \lambda_{x,j} \right)^i-1\right)
\end{eqnarray*}

We use the following two facts from~\cite{orlitsky2016optimal}. 

\paragraph{Fact 1} For all $y > 0$ and for any random variable $L$,
\[
\left|-\sum_{i=0}^{\infty} \frac{\mathbb{P}(L \ge i)}{i!} (-y)^i + e^{-y} \right| \leq \max_{s \leq y} \left| \mathbb{E}\left[ \frac{(-s)^L}{L!} \right] \right| \left( 1-e^{-y} \right).
\]

\paragraph{Fact 2} If $L \sim \mbox{Poi}(r)$, then 
\[
\left| \mathbb{E}\left[ \frac{(-s)^L}{L!} \right] \right| \leq e^{-r}.
\]

Therefore, the contribution of $x$ to $\mathbb{E}[\hat{U}^W - U]$ is 
\[
e^{-\sum_{j=1}^m \lambda_{x, j} - \sum_{j\not\in A} t_j\lambda_{x, j}} \left[ -\sum_{i=0}^{\infty} \frac{\mathbb{P}(L \geq i)}{i!} \left( -\sum_{j\in A} t_j \lambda_{x,j} \right)^i  + e^{-\sum_{j \in A} t_j \lambda_{x,j}} \right] \leq \left( 1 - e^{-\sum_{j \in A} t_j \lambda_{x,j}} \right) e^{-r}
\]
where we used Facts 1 and 2 with $y = \sum_{j \in A} t_j \lambda_{x,j}$ and the fact that $e^{-\sum_{j=1}^m \lambda_{x, j} - \sum_{j\not\in A} t_j\lambda_{x, j}} \leq 1$. 

Now summing over $x \in \mathcal{X}$, we have 
\begin{eqnarray*}
\left| \mathbb{E}[\hat{U}^W - U] \right| &\leq&  \sum_x \left( 1 - e^{-\sum_{j \in A} t_j \lambda_{x,j}} \right) e^{-r} \\
&\leq& \sum_x\left( 1 - e^{-\sum_{j \in A} (t_j+1) \lambda_{x,j}} \right) e^{-r} \\
&=& \sum_x \left[ e^{-\sum_{j \in A} \lambda_{x,j}} \left( 1 - e^{-\sum_{j \in A} t_j\lambda_{x,j}}  \right) + 1 - e^{-\sum_{j \in A} \lambda_{x,j}} \right]e^{-r} \\
&=& \left(\mathbb{E}[U^A] + \mathbb{E}[\Phi_+^A] \right) e^{-r} \\
&\leq& \left(\sum_{j\in A} n_j (t_j+1)  \right)e^{-r}
\end{eqnarray*}
where $\Phi_+^A$ is the total number of distinct elements observed in period one for subpopulations $j \in A$ and $U^A$ is the number of new elements observed in period two for $j \in A$. 

\end{proof}

Lemma~\ref{lem:bias} quantifies the bias of the weighted estimator. Next we quantify its variance.

\begin{lemma}[Lemma~\ref{lem:var} restated]
Without loss of generality, let $t_1 = \max_j t_j$ and suppose $t_1 \geq 1$ then
\[
\mbox{Var}(\hat{U}^W - U) \leq n e^{2r(t_1-1)} + \sum_j n_j t_j.
\]
\end{lemma} 

\begin{proof}
Let $N_{x,j}$ be the random variable corresponding to the number of times $x$ is found in population $j$ during period one. Let $N'_{x,j}$ be the random variable corresponding to the number of times $x$ is found in population $j$ during period two. Define $h(i_1, ..., i_m) = - \prod_{j=1}^m (-t_{i_j})^{i_j}\mathbb{P}\left(L \geq \sum_{j\in A} i_j  \right)$.

For every element $x$, its contribution to $\mbox{Var}(\hat{U}^W - U)$ is
\begin{eqnarray*}
& & \mbox{Var}\left[ \sum_{i_1, ..., i_m} \left(\prod_{j}1_{N_{x,j}=i_j}\right) h(i_1,...,i_m) -  \left(\prod_{j}1_{N_{x,j}=0}\right)\left( 1- \prod_{j}1_{N'_{x,j}=0} \right) \right] \\
&\leq& \mathbb{E} \left[ \sum_{i_1, ..., i_m} \left(\prod_{j}1_{N_{x,j}=i_j}\right) h(i_1,...,i_m) -  \left(\prod_{j}1_{N_{x,j}=0}\right)\left( 1- \prod_{j}1_{N'_{x,j}=0} \right) \right]^2 \\
&=& \mathbb{E} \left[ \sum_{i_1, ..., i_m} \left(\prod_{j}1_{N_{x,j}=i_j}\right) h(i_1,...,i_m)^2 + \left(\prod_{j}1_{N_{x,j}=0}\right)\left( 1- \prod_{j}1_{N'_{x,j}=0} \right)\right].
\end{eqnarray*}
The last equality follows because the cross-term vanishes since the events $N_{x,j}=0, \forall j$ and $N_{x,j}=i_j, \sum_j i_j > 0$ are disjoint. Summing over all $x$ gives
\begin{align}
\mbox{Var}(\hat{U}^W - U) &\leq \mathbb{E}[\Phi_+] \sup_{i_1,...,i_m}h(i_1, ..., i_m)^2 + \mathbb{E}[U] \\
&\leq n \sup_{i_1,...,i_m}h(i_1, ..., i_m)^2 + \sum_j n_j t_j \label{eqn:variance}.
\end{align}
Moreover we have
\begin{eqnarray*}
| h(i_1, ..., i_m) | &=& \left( \prod_j t_j^{i_j} \right) \mathbb{P}
\left(L \geq \sum_{j \in A} i_j  \right) \\
&\leq& t_1^{\sum_{j \in A} i_j} \mathbb{P}\left(L \geq \sum_j i_j  \right) \\
&\leq& e^{r(t_1-1)}
\end{eqnarray*}
where we have used the following fact:

\paragraph{Fact 3} If $L \sim \mbox{Poi}(r)$ and $t\geq 1$, then for all $i>0$
\[
t^i \mathbb{P}(L \geq i) \leq e^{r(t-1)}.
\]
Note that only $t_{max}\geq 1$ is assumed here; the other $t_j$'s could be less than 1. 
\end{proof}

Putting the last two lemmas together, we have
\begin{lemma}
Let $t_1 = \max_j t_j \geq 1$, then
\[
\mathbb{E}\left[ (\hat{U}^W - U)^2 \right] \leq n e^{2r(t_1-1)} + \sum_j n_j t_j + \left(\sum_{j\in A} n_j( t_j + 1) \right)^2e^{-2r}
\]
\label{lem:combined}
\end{lemma}

Because $t_1 \geq 1$, Lemma~\ref{lem:combined} implies that 
\[
\mathbb{E}\left[ (\hat{U}^W - U)^2 \right] \leq (n+\sum_j n_j t_j) e^{2r(t_1-1)}+(n+\sum_j n_j t_j)^2 e^{-2r}.
\]
The two terms on the RHS are equal when $r=\frac{\log (n+\sum_j n_j t_j)}{2 t_{max}}$. Using this value of $r$, we have
\begin{eqnarray*}
\mathbb{E}\left[ \left(\frac{\hat{U}^W - U}{\sum n_j t_j}\right)^2 \right] &\leq& \left( \frac{n + t_1\bar{n}}{t_{\*}\bar{n}} \right)^2 (n + t_1\bar{n})^{-1/t_1} \\
&\leq& \left( \frac{n + t_1\bar{n}}{t_1\bar{n}} \right)^2 \bar{n}^{-1/t_1} \\
&\leq& \left( \frac{n + t_1n_1}{t_1n_1} \right)^2 n_{1}^{-1/t_1}
\end{eqnarray*}
where $\bar{n} \equiv \sum_j n_j t_j / t_1$. This completes the proof of Prop.~\ref{prop:guarantee}.

\section{Multi-population Earth Mover's Distance}
We define a natural distance metric on multi-population histograms, which is a measure of the extent to which the corresponding distributions are similar, up to a relabeling of the elements:

\begin{definition}\label{def:emd}
Given two $m$-population histograms, $H, H'$, the \emph{multi-population earthmover distance} $d_W(H,H')$ is defined as the minimum over all schemes of moving the histogram elements in $H$ to yield $H'$, where the cost of moving $c$ histogram elements from $\boldsymbol{\alpha} \in [0,1]^m$ to $\boldsymbol{\alpha}'$ is $c\frac{1}{2m}\|\boldsymbol{\alpha}-\boldsymbol{\alpha}'\|_1=\sum_{i=1}^m |\alpha_i-\alpha'_i|.$  To ensure that such a scheme exists, we regard there being an infinitude of elements that occur with probability zero in all populations, $H(\boldsymbol{0})=H'(\boldsymbol{0})=\infty.$
\end{definition}

Note that for all pairs of histograms $H,H'$ it holds that $d_W(H,H') \in [0,1]$, with $d_W(H,H')=0$ if and only if the distributions corresponding to $H$ and $H'$ are identical, up to relabeling the domain elements.  The following example illustrates the above definition:

\begin{example}
Consider a 3 population distribution corresponding to a three uniform distributions over $2n$ elements, where $n$ of the elements are common to all 3 populations, and the other elements are unique.  This corresponds to histogram $H$ defined by $H(1/2n,1/2n,1/2n)=n$, $H(1/2n,0,0)=n$, $H(0,1/2n,0)=n$, $H(0,0,1/2n)=n$.  Consider a second histogram $H'$ corresponding to three uniform distributions over a common set of $n$ elements, $H'(1/n,1/n,1/n)=n,$ and $H'(\boldsymbol{\alpha})=0$ for all $\boldsymbol{\alpha} \neq (1/n,1/n,1/n).$  The EMD $$d_W(H,H') =\frac{1}{2\cdot 3}\left(n\frac{3}{2n}+3n\frac{1}{2n}\right)=\frac{1}{2},$$ Since we can make $H'$ from $H$ by moving $n$ histogram elements from $(1/2n,1/2n,1/2n)$ to $(1/n,1/n,1/n)$ at a per-unit-cost of $\|(\frac{1}{2n},\frac{1}{2n},\frac{1}{2n})-(\frac{1}{n},\frac{1}{n},\frac{1}{n})\|_1=\frac{3}{2n},$ and then moving the remaining $3n$ elements of $H$ to $(0,0,0)$ at a per-unit-cost of $1/2n.$
\end{example}

\section{Additional experiments}
We tested the prediction accuracy of $\hat{H}_{count}$ on a different statistic: the number of elements we expect to find at least twice in the new samples, see Fig.~\ref{fig:4pop_words2}.

\begin{figure}[!tbp]
  \includegraphics[width=0.4\textwidth]{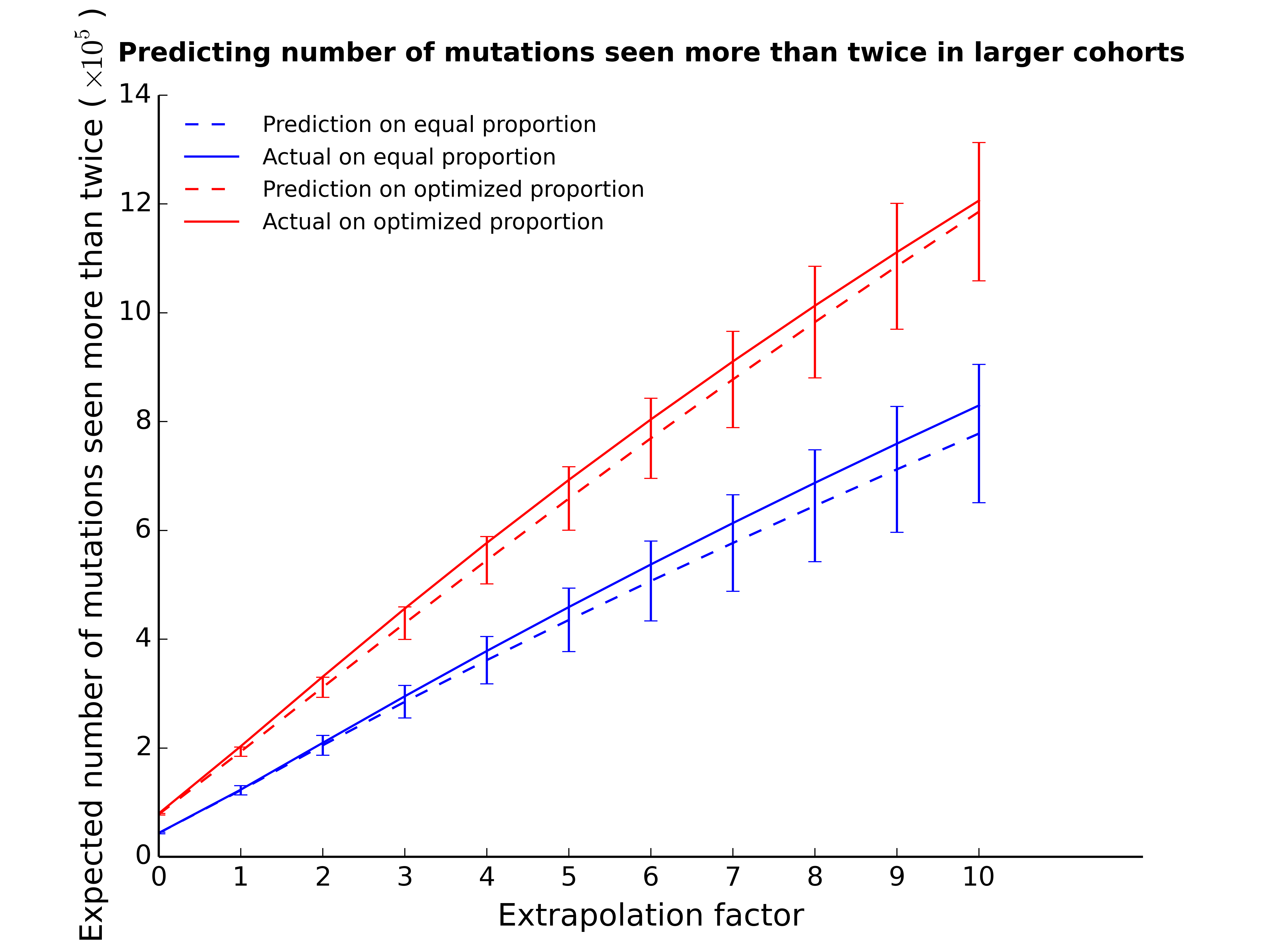}
  \caption{Estimating the number of new mutations that would be observed two or more times given additional samples from four different populations using hist-opt-counts. We consider different ratios of sampling within these four populations and observe the change in number of new mutations that would be observed for a fixed total number of new samples. }
  \label{fig:4pop_words2}
\end{figure}

\end{document}